\documentclass[conference]{IEEEtran}

\usepackage{cite}
\usepackage{booktabs}
\usepackage{amsmath}

\usepackage{amssymb,amsfonts}
\usepackage{amsthm}
\usepackage[ruled,vlined,linesnumbered]{algorithm2e}
\usepackage{graphicx}
\usepackage{multirow}
\usepackage{float}
\usepackage{url}
\usepackage{xspace}
\newcommand{\method}{GraphMNL\xspace}
\newcommand{\plusmethod}{GraphMNL$^{+}$\xspace}

\newcommand{\lossEG}{\mathcal{L}_{\mathrm{EG}}}

\newcommand{\softmax}{\operatorname{softmax}}
\newcommand{\stopgrad}{\operatorname{sg}}
\newcommand{\vect}[1]{\boldsymbol{#1}}
\newcommand{\mat}[1]{\mathbf{#1}}

\newtheorem{assumption}{Assumption}
\newtheorem{lemma}{Lemma}
\newtheorem{proposition}{Proposition}

\begin{document}

\title{Multimodal Graph Negative Learning}

\author{Zhengyu Wu, Xu Wang, Hongchao Qin, Xunkai Li, Guang Zeng, Rong-Hua Li, Guoren Wang}

\maketitle

\begin{abstract}
Multimodal attributed graphs (MAGs) integrate graph topology with heterogeneous modality attributes, such as text and images, thereby enabling richer modeling of complex relational systems. 
However, such expressiveness also makes learning on MAGs depend on multiple semantic sources, including structural topology, textual and visual attributes, each of which can be regarded as a branch for node representation. 
Node-level branch semantic imbalance arises when these branches differ across nodes in semantic informativeness and reliability: a branch that provides discriminative semantics for one node may mislead another due to bias in modality quality or structural context.
Existing methods often mitigate such heterogeneity through cross-branch agreement or alignment, implicitly treating the dominant prediction as reliable supervision. When the dominant branch is biased, forced imitation may propagate its bias to other branches and suppress original semantics that are useful for classification.
We propose GraphMNL, a graph-aware multimodal negative learning framework that addresses this issue by using Negative Learning as cross-branch guidance.
Instead of forcing inferior branches to imitate a teacher prediction, the model teaches them which classes a node is unlikely to belong to. 
GraphMNL builds a branch library, identifies dominant and inferior branches via graph-aware reliability arbitration, gates unstable transfer, and applies target-preserving negative learning over non-target classes. 
This design decouples target supervision from branch guidance so that supervised losses learn the correct class, while Negative Learning suppresses unlikely alternatives when branch agreement is unreliable.
Through the comprehensive experimental evaluation, GraphMNL achieves the best performance on Grocery datasets with 72.47\% accuracy and 76.60 F1 score on Reddit M datasets, with 1.81\% and 3.63\% improvement on the second best baseline respectively.

\end{abstract}

\begin{IEEEkeywords}
multimodal attributed graph, negative learning, graph neural network, modality stress behavior, multimodal fusion
\end{IEEEkeywords}

\section{Introduction}

Multimodal Attributed Graphs (MAGs) are graph-structured data in which nodes and edges are associated with heterogeneous modality attributes, such as textual descriptions and visual content. Unlike conventional graph datasets, which mainly characterize entities through relational topology and compact numerical attributes, MAGs jointly model structural dependencies and modality-specific semantics, thereby providing a more expressive representation of graph entities. This richer representational capacity enables MAGs to capture the heterogeneous and interdependent information patterns commonly observed in complex real-world scenarios, and has therefore supported broad applications in recommendation, product search, social media analysis, scientific discovery, and content understanding~\cite{he2017neuralcf,hou2024bridging,e_commerce1}.

However, the same semantic heterogeneity that makes MAGs expressive also gives rise to a critical challenge, which we refer to as node-level branch semantic imbalance. 
A branch denotes a distinct semantic source used to characterize a node, including structural topology, textual descriptions, and visual content. 
In practice, these branches are rarely equally informative across all nodes.
Specifically, a node may have reliable neighborhood structure but weak textual or visual semantics, whereas another node may rely more heavily on modality attributes due to sparse or noisy structural context. 

This node-dependent imbalance directly affects how branch-specific predictions should be integrated. For example, a product node is better characterized by its image than by a generic description, and a scientific node depends primarily on citation structure or textual abstracts rather than raw visual attributes. 
Thus, even under the same model architecture, the branch that provides discriminative semantics for one node may introduce misleading shortcuts for another. Branch disagreement should therefore not be treated merely as noise to be suppressed; instead, it may indicate that different branches capture different levels of target-relevant semantics, structural bias, modality availability, or local class-frequency patterns. 

Existing MAG methods often address heterogeneous branches through shared supervision, representation alignment, or agreement regularization, thereby encouraging branch-specific predictors to converge toward a common prediction pattern~\cite{wang2019mmgcn,tao2020mgat,jia2023mhgat,guo2024lgmrec,he2025unigraph2}. While such consistency constraints can improve cross-branch coordination, they implicitly assume that the dominant branch provides reliable supervision for the others. This assumption is problematic in MAGs, where a confident prediction may arise from topology shortcuts, modality imbalance, or local class-prior bias rather than target-relevant semantics. In such cases, enforcing consistency may transfer biased semantics across branches and suppress node-specific discriminative cues. This motivates the central research question: \textit{\textbf{How to selectively leverage reliable branch-specific semantics while avoiding the propagation of biased dominant predictions across heterogeneous branches?}}

Recent advances~\cite{wang2019mmgcn,tao2020mgat,guo2024lgmrec,jia2023mhgat,guo2025dmgc,zheng2025dgf,he2025unigraph2,ning2025graph4mm} in MAG learning often overlook whether the transferred semantics are actually reliable at the node level. 
Consequently, these methods still depend on implicit consistency assumptions across branches, making them vulnerable to propagating biased semantics. 
In particular, current approaches face two fundamental training limitations. \textbf{(1) Reckless Imitation on Dominant Branches.} Agreement-based training may be beneficial in standard multimodal settings, but in MAGs it can force a semantically disadvantaged branch to imitate a dominant branch whose confidence is driven by structural bias rather than robust semantic discrimination~\cite{mcpherson200homophily_theory1,M2003Mixing_homophily_theory2,ma2021hete_gnn_survey1,luan2022hete_gnn_survey2,2020h2gcn,2021linkx}, thereby overwriting original semantics that are essential for identifying the target class. \textbf{(2) Node-wise Branch Reliability Mismatch.} Branch reliability often varies substantially across nodes due to corrupted modality inputs, even when such branches still provide highly discriminative semantics for a specific subset of nodes. Branch reliability is therefore not determined by confidence alone, but depends jointly on factors such as class separability, modality availability, neighborhood consistency, and clear distinction from alternative prediction paths~\cite{gong2025mnl,ma2021smil,ma2022missingtransformer,wu2024missingmodalitysurvey}. Consequently, fixed fusion or confidence-based strategies cannot reliably determine branch supervision and semantic transfer.

To address these limitations, we propose \textbf{GraphMNL}, a graph-aware multimodal negative learning framework that converts branch disagreement into reliability-aware negative supervision. Rather than asking an inferior branch to imitate the full prediction of a potentially biased dominant branch, GraphMNL uses Negative Learning to transfer what a node is unlikely to be, while leaving the target class to supervised objectives. Its design follows two principles: \textit{\textbf{Branch reliability should be identified adaptively for each node, and cross-branch transfer should be negative, target-preserving, and robust to biased dominant predictions.}} GraphMNL realizes these principles through four connected modules. First, branch construction disentangles heterogeneous prediction pathways, including raw-modality, graph-structural, and modality-propagated branches, allowing the model to inspect distinct structural and modality-specific semantics instead of collapsing them into a single fused representation. Second, graph-aware reliability arbitration selects a dominant branch only when confidence, class margin, neighborhood consistency, modality availability, and training stability jointly support its prediction. Third, stability gating transforms reliability gaps into dynamic transfer weights, preventing unstable or locally unreliable branches from dominating guidance. Finally, graph-aware exclusion guidance implements negative learning over non-target classes, suppressing unlikely alternatives without requiring full posterior imitation. Experiments show that \method achieves the best performance consistently over competitive baselines with 92.47\% accuracy on the Reddit S datasets and achieve the best efficiency-performance tradeoffs.

Our contributions are: 
\textit{\textbf{(1) Novel Perspective.}} We identify MAG imbalance as both an input quality issue and a prediction reliability issue, and we show why branch disagreement should be used as local training evidence rather than always removed through positive imitation.
\textit{\textbf{(2) New Paradigm.}} We introduce GraphMNL and GraphMNL+, where GraphMNL provides exclusion based branch guidance, while GraphMNL+ further strengthens node adaptive reliability estimation and stability controlled exclusion over non target classes.
\textit{\textbf{(3) SOTA Performance.}} \method demonstrate the best node classification performance with consistent gains from the competitive baselines. \plusmethod exhibit the superior performance and consistence acorss various datasets and tasks. The improvement is supported by reliable branch selection and exclusion guidance, which strengthen hard class discrimination without relying on dense infer objectives.

\section{Related Work}

\subsection{Multimodal Attributed Graph Learning}
MAG benchmarks require topology and multimodal attributes to be evaluated together. MAG studies multimodal attributed graph learning across representative graph tasks~\cite{yan2024magb}, while MM Graph broadens benchmark coverage for multimodal graph learning~\cite{zhu2025mmgraph}. These benchmarks show that useful MAG models must support both structural reasoning and semantic evidence from text, images, and metadata.

MAG model families are especially useful for positioning GraphMNL. Graph enhanced methods disentangle homophilic and heterophilic multimodal graph signals or apply graph filtering and denoising~\cite{guo2025dmgc,zheng2025dgf}. Multimodal enhanced methods build modality specific graph convolution, modality aware attention, heterogeneous multimodal attention, or local global graph learning~\cite{wang2019mmgcn,tao2020mgat,jia2023mhgat,guo2024lgmrec}. Foundation style methods further connect graph structure with large multimodal encoders through unified embeddings, graph conditioned multimodal learning, multimodal language and graph assistants, graph comprehension and generation models, and self teaching graph transformers~\cite{he2025unigraph2,ning2025graph4mm,MLaGA,GraphGPT-O,NTSFormer}.

These methods strengthen encoders, but they mostly answer how to represent graph and semantic evidence rather than how branches should teach one another when evidence conflicts. GraphMNL therefore keeps branch predictions explicit, estimates reliability for each node, and transfers only exclusion knowledge from reliable branches.

\subsection{Modality Imbalance and Negative Learning}
Incomplete modality and modality imbalance methods show why agreement across views can be unreliable. SMIL, Missing Modality Transformer, ShaSpec, and MMANet handle absent or incomplete channels~\cite{ma2021smil,ma2022missingtransformer,wang2023shaspec,zhao2023mmanet}, while federated multimodal methods address distributed modality heterogeneity through foundation models, cross modal aggregation, contrastive regularization, or computational pathology settings~\cite{che2024fedmvp,peng2024fedmm}.

Negative learning is especially relevant because it teaches what a sample is unlikely to be, originating from complementary label learning where classifiers are trained to avoid specified incorrect classes rather than to predict the correct one~\cite{ishida2017complementary,yu2018multiplecl}. The paradigm was later extended to noisy-label robustness under the name negative learning, where suppressing unlikely classes reduces the risk of overfitting to corrupted labels~\cite{kim2019nlnl}. Most recently, multimodal negative learning (MNL)~\cite{gong2025mnl} generalized this idea to multimodal settings with missing or unreliable modalities. MAGs add risks specific to graphs because propagation, degree, homophily, and neighborhoods can make a branch confident for structural rather than semantic reasons. GraphMNL therefore makes negative learning graph aware through reliability based role assignment, stability gating, and guidance only over classes outside the target.

\section{Preliminaries}

This section fixes the notation used by GraphMNL for MAG node classification, branch centered prediction paths, graph aware reliability, and exclusion guidance. It is a compact contract for the Method section rather than a separate background survey.

\subsection{MAG Task Setting}

Following MAG benchmarks~\cite{openmag2026,zhu2025mmgraph,yan2024magb}, a multimodal attributed graph is $\mathcal{G}=(\mathcal{V},\mathcal{E},\mat{A},\allowbreak\{\mat{X}^{m}\}_{m\in\mathcal{M}},\allowbreak\{\mat{Z}^{r}\}_{r\in\mathcal{R}})$, where $\mathcal{V}$ is the node set, $\mathcal{E}$ is the edge set, $\mat{A}$ is the adjacency matrix, $\mat{X}^{m}$ stores node attributes for modality $m$, and $\{\mat{Z}^{r}\}_{r\in\mathcal{R}}$ denotes optional edge features. The active task is node classification with labeled nodes $\mathcal{V}_L$, unlabeled nodes $\mathcal{V}_U$, and labels $y_i\in\{1,\ldots,C\}$.

The setting is supervised in its loss definition and semi supervised in graph usage. Labels in $\mathcal{V}_L$ provide target supervision and exclusions for classes outside the target, while nodes in $\mathcal{V}_U$ do not directly contribute supervised or exclusion guidance terms. However, unlabeled nodes remain in the observed graph, so their features and edges can still affect propagation, neighborhood consistency, and branch reliability.

\subsection{Branch Centered Prediction Paths}

GraphMNL uses branches to keep each information view explicit before transfer between branches is applied. A branch is a task supervised prediction path built from one view of $\mathcal{G}$, such as raw text or image features, graph topology, modality propagated features, cross modal paths, or frequency domain propagation. The auxiliary branch set $\mathcal{B}$ contains the unfused prediction paths used for branch supervision and reliability comparison, while the fused path combines branch evidence and produces the main prediction $\vect{p}_i^{\mathrm{fuse}}$. For node $i$, each branch $b\in\mathcal{B}$ produces a hidden representation $\vect{h}_i^b$, logits $\vect{o}_i^b\in\mathbb{R}^{C}$, and class probabilities:
\begin{equation}
    \vect{p}^{b}_{i} = \softmax(\vect{o}^{b}_{i}).
\end{equation}
These probabilities expose each branch as evidence provider, supervised predictor, candidate teacher, or guided branch. This role separation matters because GraphMNL does not assume the most confident branch is always the best teacher.

\subsection{Graph Aware Reliability Interface}

For each node and branch, GraphMNL summarizes prediction evidence and graph context before assigning teacher roles. The predicted class is $\hat{y}_i^b=\arg\max_c(\vect{p}_i^b)_c$, the confidence is $q_i^b=\max_c(\vect{p}_i^b)_c$, and the top two gap $\Delta_i^b$ measures separation from the nearest competing class. Reliability also depends on neighborhood consistency $\eta_i^b$, modality availability $a_i^b$, temporal stability $\sigma_i^b$, and an uncertainty penalty $u_i^b$ that can reflect degree effects, prediction variance, and disagreement with the fused path. These cues form a graph aware reliability score $\rho_i^b$ in Method.

Teacher eligibility is label aware on $\mathcal{V}_L$. Let $\mathcal{C}_i=\{b\in\mathcal{B}:\hat{y}_i^b=y_i\}$ be the candidate teacher set for labeled node $i$. Branches outside $\mathcal{C}_i$ may still learn from supervision, but they cannot serve as dominant teachers for exclusion guidance on that node.

\subsection{Learning Objectives through GraphMNL's pipeline}

GraphMNL proposes a new training paradigm that constructs auxiliary branches from the branch library above and uses the fused path as the main task predictor. Its auxiliary guidance objective transfers only exclusion rankings over classes outside the target from reliable branches to inferior branches.

\textbf{Supervised target learning.} The fused predictor and auxiliary branches learn the target class from labeled nodes:
\begin{equation}
    \mathcal{L}_{\mathrm{sup}}
    =
    \sum_{i\in\mathcal{V}_L}
    \ell_{\mathrm{CE}}(\vect{p}_i^{\mathrm{fuse}},y_i)
    +
    \mu\sum_{b\in\mathcal{B}}
    \sum_{i\in\mathcal{V}_L}
    \ell_{\mathrm{CE}}(\vect{p}_i^b,y_i).
\end{equation}
This term defines what each labeled node should be and prevents auxiliary branches from becoming passive components of the fused head.

\textbf{Role assignment and exclusion transfer.} For each labeled node with $\mathcal{C}_i\neq\emptyset$, GraphMNL selects a dominant branch $b_i^+$ from eligible teachers and collects inferior branches $\mathcal{I}_i$ whose reliability is sufficiently lower. Exclusion guidance transfers only over the class set outside the target, written as $\mathcal{C}_i^-=\{1,\ldots,C\}\setminus\{y_i\}$, so cross entropy learns target evidence while the teacher supplies an exclusion ranking. Each guided branch receives a dynamic weight $\lambda_i^b(t)$, allowing transfer to depend on reliability gap and teacher stability.

\textbf{Overall objective.} The node classification loss combines supervised learning and graph aware exclusion guidance:
\begin{equation}
    \mathcal{L}
    =
    \mathcal{L}_{\mathrm{sup}}
    +
    \sum_{i\in\mathcal{V}_L}
    \sum_{b\in\mathcal{I}_i}
    \lambda_i^b(t)\lossEG(i,b).
\end{equation}
The objective separates the roles of the losses. Cross entropy learns target decisions for the fused predictor and each branch, while exclusion guidance shapes only alternatives outside the target for inferior branches. The Method section instantiates reliability computation, stable role assignment, and softened distributions over classes outside the target.

\begin{figure*}[t]
\centering
\includegraphics[width=0.95\textwidth]{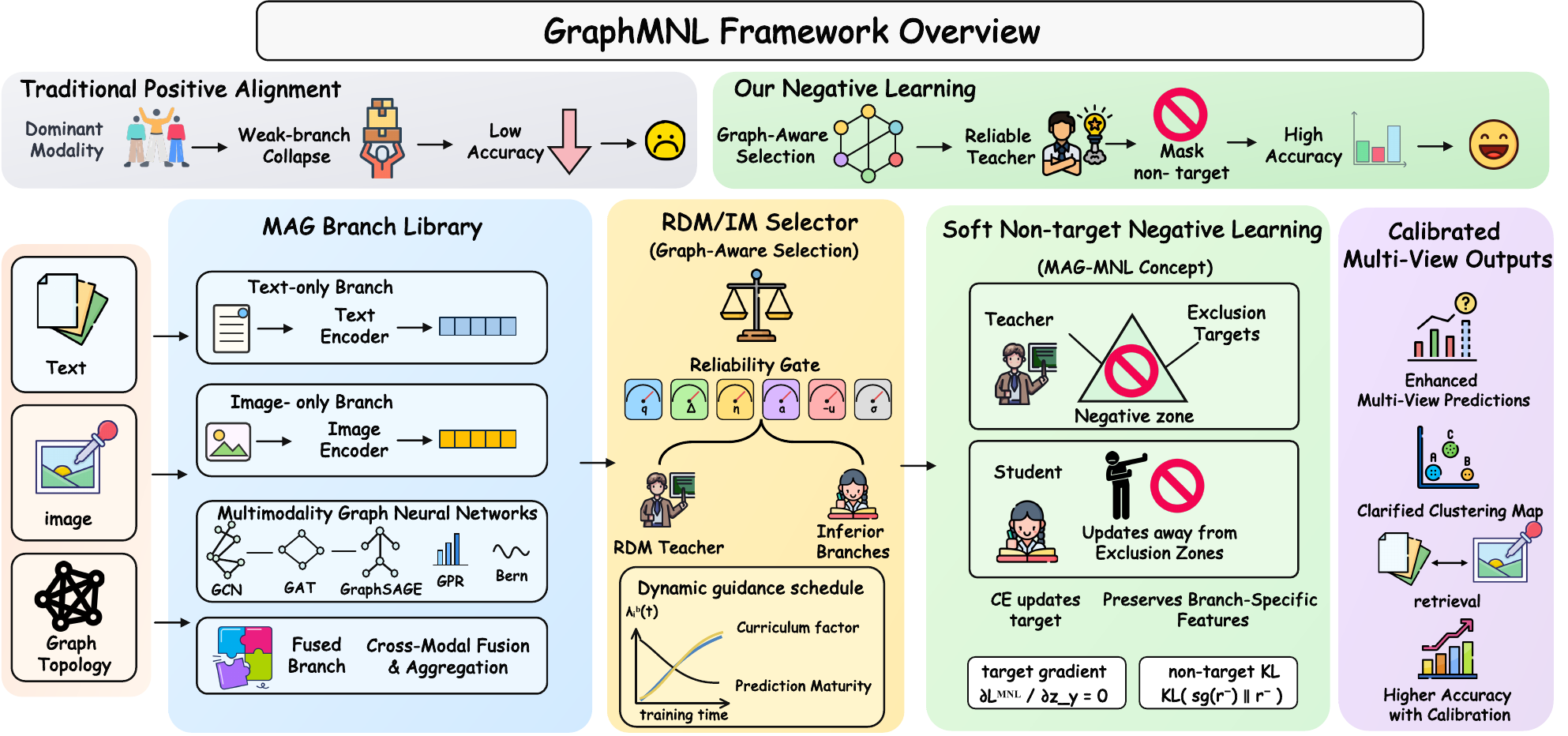}
\caption{Overview of GraphMNL. The framework contrasts positive alignment with negative learning, builds a modality and graph branch library, selects reliable teachers through graph aware RDM and IM cues, applies soft non target negative learning, and produces calibrated multi view predictions.}
\label{fig:framework}
\end{figure*}

\section{Method}

GraphMNL resolves the two limitations identified in the Introduction, unsafe positive imitation (L1) and node wise branch reliability mismatch (L2). As summarized in Fig.~\ref{fig:framework}, GraphMNL replaces full posterior imitation with exclusion based transfer. Reliable branches shape only alternatives outside the target, so target evidence from an inferior branch is not overwritten by a dominant branch that may be locally biased.

The method follows a role first pipeline. \textbf{Branch construction} exposes MAG prediction paths. \textbf{Reliability arbitration} chooses correct teachers supported by graph context. \textbf{Stability gated guidance} converts reliability gaps into transfer weights for each node. \textbf{Exclusion guidance} transfers only softened exclusion rankings over classes outside the target. Theoretical analysis then explains private evidence collapse, target preservation, and bounded graph aware guidance.

\subsection{Branch Architecture}

MAGs contain information sources whose reliability changes across nodes, so GraphMNL starts from a branch library rather than a single fused encoder. An auxiliary branch may use raw modality features, graph topology alone, modality propagated features, cross modal paths, or frequency domain propagation, while the fused path combines branch evidence for the main task prediction. This design is compatible with multimodal GNN backbones such as MMGCN, MGAT, LGMRec, and MHGAT~\cite{wang2019mmgcn,tao2020mgat,guo2024lgmrec,jia2023mhgat}. The enhanced plug and play implementation reported in Q1 uses supervised graph model and frequency domain branches over the same RoBERTa/CLIP MAG inputs and sampled splits.

Given node $i$, each branch $b\in\mathcal{B}$ outputs $\vect{h}_i^b$, $\vect{o}_i^b$, and $\vect{p}_i^b$, while the fused branch outputs $\vect{p}_i^{\mathrm{fuse}}$. Explicit branches make reliability comparison possible for each node and each branch. Later controlled comparisons keep the branch library fixed when testing whether the guidance rule between branches changes behavior.

\subsection{Reliable Branch Selection}

The key challenge in branch selection is that confidence alone is not a trustworthy teacher signal in MAGs. A graph branch can be confident because of homophily, node degree, or local leakage, while a modality branch can be uncertain even when it still contains private evidence. GraphMNL therefore allows a branch to teach only when confidence is supported by class separation, neighborhood consistency, modality availability, and temporal stability. Eq.~\eqref{eq:prediction-evidence} extracts the predicted class, confidence, and top two probability gap from each branch.
\begin{equation}
\begin{aligned}
    \hat{y}_i^b
    &=
    \arg\max_c(\vect{p}_i^b)_c,\\
    q_i^b
    &=
    \max_c(\vect{p}_i^b)_c,\\
    \Delta_i^b
    &=
    (\vect{p}_i^b)_{c_1^b}-(\vect{p}_i^b)_{c_2^b}.
\end{aligned}
\label{eq:prediction-evidence}
\end{equation}
Here $c_1^b$ and $c_2^b$ denote the top two classes under branch $b$. The value $q_i^b$ captures absolute belief, while $\Delta_i^b$ captures separation from the nearest competing class.

GraphMNL then augments prediction evidence with graph aware reliability cues in Eq.~\eqref{eq:graph-aware-cues}:
\begin{equation}
\begin{gathered}
    \eta_i^b
    =
    1-\frac{1}{|\mathcal{N}(i)|}
    \sum_{j\in\mathcal{N}(i)}
    \operatorname{JS}(\vect{p}_i^b,\vect{p}_j^b),
    \\[2pt]
    u_i^b
    =
    \omega_1\frac{1}{\log(2+d_i)}
    +\omega_2\operatorname{Var}_t[\vect{p}_{i,t}^b]
    +\omega_3\operatorname{JS}(\vect{p}_i^b,\vect{p}_i^{\mathrm{fuse}}).
\end{gathered}
\label{eq:graph-aware-cues}
\end{equation}
Here $\operatorname{JS}(\cdot,\cdot)$ denotes Jensen Shannon divergence, $\mathcal{N}(i)$ is the neighborhood of node $i$, and $d_i$ is its degree. The term $\eta_i^b$ favors neighborhood consistent predictions, while $u_i^b$ penalizes low degree, temporal variance, and disagreement with the fused branch. These cues are combined into the final reliability score:
\begin{equation}
    \rho_i^b
    =
    \alpha q_i^b+\beta\Delta_i^b+\gamma\eta_i^b+\delta a_i^b-\xi u_i^b .
\label{eq:reliability}
\end{equation}
The availability term $a_i^b$ records whether the branch input is present and informative. The coefficients control how much each reliability source contributes.
All reliability terms are computed per node and branch, so role assignment can change across the graph instead of imposing a global teacher.

For labeled nodes, GraphMNL only allows correct branches to teach. Let $\mathcal{C}_i=\{b\in\mathcal{B}:\hat{y}_i^b=y_i\}$ be the candidate teacher set. If no branch predicts the label correctly, the node is skipped for exclusion guidance. Otherwise, teacher and student roles are determined by the arbitration rule in Eq.~\eqref{eq:role-assignment}:
\begin{equation}
\begin{gathered}
    b_i^+
    =
    \arg\max_{b\in\mathcal{C}_i}\rho_i^b,
    \\[2pt]
    \mathcal{I}_i
    =
    \{b\in\mathcal{B}\setminus\{b_i^+\}:
    \rho_i^b+\tau<\rho_i^{b_i^+}\}.
\end{gathered}
\label{eq:role-assignment}
\end{equation}
The candidate set $\mathcal{C}_i$ prevents a wrong but confident branch from becoming the teacher, $b_i^+$ selects the most reliable correct branch, and $\mathcal{I}_i$ collects branches whose reliability is sufficiently lower.
This design is stronger than guidance using confidence alone because it can reject a teacher whose confidence is not supported by margin, graph consistency, modality availability, or structural stability. The experiments separate graph aware teacher selection from full posterior positive imitation and from exclusion guidance using confidence alone.
\begin{algorithm}[!b]
\SetAlgoLined
\caption{GraphMNL training pipeline.}
\label{alg:branch-selection}
\KwIn{MAG $\mathcal{G}$, labeled nodes $\mathcal{V}_L$, labels $\{y_i\}$, branch set $\mathcal{B}$, fused predictor, training horizon $T$, hyperparameters $T_n,T_{\rho},T_s,T_{\sigma},\lambda_{\max},\tau$}
\KwOut{Trained GraphMNL model}
\For{$t=1,\ldots,T$}{
    forward each branch $b\in\mathcal{B}$ to obtain $\{\vect{h}_i^b,\vect{o}_i^b,\vect{p}_i^b\}$ and fused prediction $\vect{p}_i^{\mathrm{fuse}}$\;
    initialize $\mathcal{L}$ with fused and branch supervised losses in Eq.~\eqref{eq:node-eg}\;
    \ForEach{$i\in\mathcal{V}_L$}{
        compute prediction evidence $\hat{y}_i^b$, $q_i^b$, and $\Delta_i^b$ by Eq.~\eqref{eq:prediction-evidence}\;
        compute graph-aware cues $\eta_i^b$, $a_i^b$, and $u_i^b$ by Eq.~\eqref{eq:graph-aware-cues}\;
        compute reliability $\rho_i^b$ for each branch by Eq.~\eqref{eq:reliability}\;
        assign dominant branch $b_i^+$ and inferior set $\mathcal{I}_i$ by Eq.~\eqref{eq:role-assignment}\;
        \If{$\mathcal{I}_i\neq\emptyset$}{
            compute $\lambda_i^b(t)$ for each $b\in\mathcal{I}_i$ by Eq.~\eqref{eq:dynamic-lambda}\;
            construct detached soft negative teacher $\bar{\vect{r}}_i^+$ and branch distributions $\vect{r}_i^b$ over $\mathcal{C}_i^-$ by Eq.~\eqref{eq:soft-negative-teacher}\;
            add $\sum_{b\in\mathcal{I}_i}\lambda_i^b(t)\lossEG(i,b)$ to $\mathcal{L}$ as in Eq.~\eqref{eq:node-eg}\;
        }
    }
    update GraphMNL parameters by minimizing $\mathcal{L}$ in Eq.~\eqref{eq:node-eg}\;
}
\Return{Trained GraphMNL model}\;
\end{algorithm}
\subsection{Stability Gated Guidance}

Even a high reliability branch can be noisy early in training, so GraphMNL does not apply exclusion guidance with a fixed warm up rule. Instead, it uses the curriculum factor $s(t)=1-\exp(-t/T_s)$ and computes a dynamic guidance weight for every node and inferior branch:
\begin{equation}
\begin{gathered}
    \lambda_i^b(t)
    =
    \lambda_{\max}s(t)
    \operatorname{sigmoid}
    \left(\frac{\rho_i^{b_i^+}-\rho_i^b-\tau}{T_\rho}\right)
    \exp\left(-\frac{\sigma_i^{b_i^+}}{T_\sigma}\right),
    \\[2pt]
    \sigma_i^b
    =
    \operatorname{EMA}\left(\|\vect{p}_{i,t}^{b}-\vect{p}_{i,t-1}^{b}\|_2^2\right).
\end{gathered}
\label{eq:dynamic-lambda}
\end{equation}
The reliability gap determines whether branch $b$ needs guidance, $s(t)$ prevents abrupt early transfer, and the stability term suppresses teachers whose predictions fluctuate across epochs. Here $T_\rho$ controls how sharply the reliability gap is activated, $T_\sigma$ controls the stability penalty, and $\lambda_{\max}$ caps transfer strength. This gate delays unstable cases for each node and branch while letting reliable nodes benefit earlier. The stress tests evaluate whether this graph aware guidance remains useful under heterogeneous branch reliability.

\subsection{Exclusion Guidance over Classes outside the Target}

The exclusion guidance module solves the transfer problem left by branch selection. Once a reliable teacher and inferior branches are identified, the model must guide alternatives without overwriting target evidence. GraphMNL restricts transfer to the class set outside the target, $\mathcal{C}_i^-$, making guidance between branches a ranking problem among classes the node should not take. Thus, the teacher provides an exclusion ranking, the student keeps its target path, and the gate decides whether the ranking is reliable enough to matter. For each class $c\in\mathcal{C}_i^-$ outside the target, Eq.~\eqref{eq:soft-negative-teacher} normalizes logits on those classes and softens the teacher.
\begin{equation}
\begin{aligned}
    r_{i,c}^{+}
    &=
    \frac{\exp((\vect{o}_i^{b_i^+})_c/T_n)}
    {\sum_{k\in\mathcal{C}_i^-}\exp((\vect{o}_i^{b_i^+})_k/T_n)},\\
    r_{i,c}^{b}
    &=
    \frac{\exp((\vect{o}_i^{b})_c/T_n)}
    {\sum_{k\in\mathcal{C}_i^-}\exp((\vect{o}_i^{b})_k/T_n)},\\
    \bar{r}_{i,c}^{+}
    &=
    (1-\varepsilon_n)r_{i,c}^{+}
    +\frac{\varepsilon_n}{|\mathcal{C}_i^-|}.
\end{aligned}
\label{eq:soft-negative-teacher}
\end{equation}
Here $T_n$ is the temperature over classes outside the target, $\varepsilon_n$ is the smoothing floor, and $\bar{\vect{r}}_i^+$ is detached before supervising inferior branches. These two smoothing controls avoid sharp spikes among excluded classes or collapse onto a single negative class. Since the target class is excluded, the teacher corrects unlikely alternatives while leaving target evidence to supervised learning.

The full node classification objective then combines fused supervision, branch supervision, and exclusion guidance:
\begin{equation}
\begin{gathered}
\begin{aligned}
    \mathcal{L}
    &=
    -\sum_{i\in\mathcal{V}_L}\log(\vect{p}_i^{\mathrm{fuse}})_{y_i}
    -\mu\sum_{b\in\mathcal{B}}\sum_{i\in\mathcal{V}_L}
    \log(\vect{p}_i^b)_{y_i} \\
    &\quad+
    \sum_{i\in\mathcal{V}_L}
    \sum_{b\in\mathcal{I}_i}
    \lambda_i^b(t)
    \lossEG(i,b)
    +\Omega,
\end{aligned}
\\[2pt]
\lossEG(i,b)
=
T_n^2
\operatorname{KL}
\left(
\stopgrad(\bar{\vect{r}}_i^+)
\middle\|
\vect{r}_i^b
\right).
\end{gathered}
\label{eq:node-eg}
\end{equation}
The first two terms learn target decisions, while $\lossEG$ shapes only alternatives outside the target through a detached and dynamically weighted teacher ranking. The term $\Omega$ denotes ordinary regularization. Positive KL and exclusion guidance using confidence alone test this target preserving distinction.

\subsection{GraphMNL$^{+}$ Plug and Play Implementation}

\plusmethod is the enhanced plug and play implementation of GraphMNL used for the main whole graph MAG node classification comparison. It keeps Eq.~\eqref{eq:node-eg} and instantiates the branch library with multimodal graph backbones, including MMGCN\cite{wang2019mmgcn}, MGAT\cite{tao2020mgat}, MHGAT\cite{jia2023mhgat}, LGMRec\cite{guo2024lgmrec}, DMGC\cite{guo2025dmgc}, and DGF\cite{zheng2025dgf} style prediction paths over the same RoBERTa and CLIP MAG inputs used by the experimental protocol. Training combines class balanced supervision with graph aware exclusion guidance, so each multimodal branch learns the target task while reliable branches provide non target exclusion knowledge. 

\section{Theoretical Analysis}

\subsection{Assumptions and Analysis Questions}

This section completes the Method core by summarizing GraphMNL's local update behavior. Lemma~\ref{lem:private-collapse} explains why positive alignment can damage private branch evidence, Proposition~\ref{prop:target-preserving} shows that exclusion guidance preserves target evidence, and Proposition~\ref{prop:stability-control} shows that stability gating suppresses unreliable transfer. These claims motivate the controls. Positive KL tests full positive imitation, exclusion guidance using confidence alone tests teacher selection without graph aware reliability, and matched removal tests whether exclusion guidance contributes beyond the enhanced implementation.

\begin{assumption}
\label{assump:local-update}
For each labeled node $i$, branch logits are differentiable, the dominant branch $b_i^+$ is detached as teacher, exclusion guidance applies only to $b\in\mathcal{I}_i$ over $\mathcal{C}_i^-$, the softened teacher in Eq.~\eqref{eq:soft-negative-teacher} is bounded away from zero, and $0\le\lambda_i^b(t)\le\lambda_{\max}$.
\end{assumption}

\subsection{Private Evidence Collapse under Positive Alignment}

Positive alignment treats disagreement as an error to remove, which can be unsafe when an inferior branch contains modality private evidence absent from the teacher.

\begin{lemma}
\label{lem:private-collapse}
For a positive feature alignment loss
\begin{equation}
    \mathcal{L}_{\mathrm{pa}} =
    \frac{1}{2}
    \|\mat{P}_{b}\vect{h}^{b}
    -\stopgrad(\mat{P}_{t}\vect{h}^{t})\|_2^2 .
\label{eq:positive-alignment}
\end{equation}
where $\mat{P}_b$ and $\mat{P}_t$ are projection heads for branch $b$ and teacher $t$. Let $\vect{u}$ be a branch private direction with $\|\vect{u}\|_2=1$. If the teacher target has no component along $\mat{P}_b\vect{u}$, then gradient descent on Eq.~\eqref{eq:positive-alignment} decreases the component of $\vect{h}^b$ along $\vect{u}$ whenever $\mat{P}_b^{\top}\mat{P}_b\vect{u}$ is positively aligned with $\vect{u}$.
\end{lemma}

\noindent\textit{Proof sketch.}
The gradient is $\nabla_{\vect{h}^{b}}\mathcal{L}_{\mathrm{pa}}=\mat{P}_{b}^{\top}(\mat{P}_{b}\vect{h}^{b}-\mat{P}_{t}\vect{h}^{t})$. After one gradient step, the private component changes by
\begin{equation}
    \langle \vect{h}_{+}^{b},\vect{u}\rangle
    =
    \langle \vect{h}^{b},\vect{u}\rangle
    -\eta
    \langle\nabla_{\vect{h}^{b}}\mathcal{L}_{\mathrm{pa}},
    \vect{u}\rangle .
\label{eq:private-component-update}
\end{equation}
If the teacher lacks the private component, the second inner product in Eq.~\eqref{eq:private-component-update} is dominated by $\langle\mat{P}_{b}^{\top}\mat{P}_{b}\vect{h}^{b},\vect{u}\rangle$, so positive alignment shrinks branch private evidence.

\subsection{Target Preserving Property of Exclusion Guidance}

Exclusion guidance protects target evidence because it applies transfer only on classes outside the target.

\begin{proposition}
\label{prop:target-preserving}
For node $i$ and inferior branch $b$, define
\begin{equation}
    \ell_{i,b}
    =
    T_n^2\lambda_i^b(t)
    \operatorname{KL}
    \left(
    \stopgrad(\bar{\vect{r}}_i^+)
    \middle\|
    \vect{r}_i^b
    \right),
\label{eq:local-eg-loss}
\end{equation}
where $\vect{r}_i^b$ is computed only on $\mathcal{C}_i^-$. Then Eq.~\eqref{eq:local-eg-loss} gives
\begin{equation}
    \partial_{(\vect{o}_i^b)_{y_i}}\ell_{i,b}=0,\quad
    \|\nabla_{\vect{o}_{i,\mathcal{C}_i^-}^{b}}\ell_{i,b}\|_2
    \le T_n\lambda_i^b(t)\sqrt{2}.
\label{eq:eg-gradient-bound}
\end{equation}
\end{proposition}

\noindent\textit{Proof sketch.}
The target logit is excluded from the restricted softmax, so its derivative is zero. For logits outside the target, temperature scaled KL gives
\begin{equation}
    \nabla_{\vect{o}_{i,\mathcal{C}_i^-}^{b}}\ell_{i,b}
    =
    T_n\lambda_i^b(t)
    (\vect{r}_i^b-\bar{\vect{r}}_i^+).
\label{eq:eg-kl-gradient}
\end{equation}
Both distributions lie on the simplex, so the right hand side of Eq.~\eqref{eq:eg-kl-gradient} has Euclidean norm at most $T_n\lambda_i^b(t)\sqrt{2}$, yielding Eq.~\eqref{eq:eg-gradient-bound}. Supervised learning therefore remains responsible for the correct class, while exclusion guidance reshapes unlikely alternatives.

\subsection{Stability Control under Graph Aware Guidance}

Because confidence can reflect degree effects, structural bias, or unstable neighborhoods, the stability gate turns reliability into a bounded transfer rule.

\begin{proposition}
\label{prop:stability-control}
Let $g_i^b=\rho_i^{b_i^+}-\rho_i^b-\tau$ be the reliability gap used by Eq.~\eqref{eq:dynamic-lambda}. For every inferior branch $b$, the guidance weight satisfies
\begin{equation}
\begin{aligned}
    0\le \lambda_i^b(t)
    &\le
    \lambda_{\max}s(t)
    \exp\left(-\frac{\sigma_i^{b_i^+}}{T_\sigma}\right)
    \le \lambda_{\max},\\
    \partial_{g_i^b}\lambda_i^b(t)
    &\ge 0,\quad
    \partial_{\sigma_i^{b_i^+}}\lambda_i^b(t)
    = -\lambda_i^b(t)/T_\sigma \le 0 .
\end{aligned}
\label{eq:stability-bound}
\end{equation}
\end{proposition}

\noindent\textit{Proof sketch and takeaway.}
The sigmoid, curriculum factor, and exponential stability term in Eq.~\eqref{eq:dynamic-lambda} are nonnegative and at most one. The sigmoid is monotone in the reliability gap, and the exponential term decreases with teacher variance. Therefore Eq.~\eqref{eq:stability-bound} gives a bounded, monotone transfer rule. Branch supervision learns targets, exclusion guidance reshapes only alternatives outside the target, and graph aware gating prevents unstable teachers from dominating inferior branches.

\noindent\textbf{Theoretical takeaways.}
The analysis shows that GraphMNL is not merely a branch ensemble. Positive alignment can remove branch private components, exclusion guidance gives zero gradient to the target logit and bounded gradients to logits outside the target, and graph aware gating increases guidance only when the reliability gap is large and the teacher is stable. Thus branch supervision preserves target evidence, exclusion guidance transfers exclusion knowledge, and gating limits destructive imitation. The theory supports the mechanism tests in Q2 and the matched ablation in Q3, but it does not by itself claim universal robustness or broad benchmark dominance.

\begin{table*}[!t]
\centering
\small
\setlength{\tabcolsep}{5pt}
\renewcommand{\arraystretch}{1.05}
\caption{MAG datasets' statistics with RoBERTa/CLIP features.}
\label{tab:dataset-stats}
\resizebox{0.9\textwidth}{!}{%
\begin{tabular}{@{}llrrrccccc@{}}
\toprule
Dataset & Domain & Nodes & Edges & Classes & Modalities & Split & Avg. Words & Image Res. & Feat. Dim.\\
\midrule
Movies & Online commerce & 16,672 & 218,390 & 20 & Text, Visual & 60/20/20 & 81.85 & 388$\times$476 & 768/768\\
Grocery & Online commerce & 17,074 & 171,340 & 20 & Text, Visual & 60/20/20 & 67.36 & 402$\times$457 & 768/768\\
Toys & Online commerce & 20,695 & 126,886 & 18 & Text, Visual & 60/20/20 & 74.50 & 467$\times$442 & 768/768\\
Reddit S & Social media & 15,894 & 566,160 & 20 & Text, Visual & 60/20/20 & 10.23 & 2515$\times$2399 & 768/768\\
Reddit M & Social media & 99,638 & 1,167,188 & 50 & Text, Visual & 60/20/20 & 10.22 & 2605$\times$2662 & 768/768\\
\bottomrule
\end{tabular}
}
\end{table*}

\section{Experiments}

\noindent
This section evaluates GraphMNL through effectiveness, stress behavior, ablation evidence, and parameter sensitivity. The evaluation answers four questions. \textbf{Q1 Effectiveness:} Whether GraphMNL improves the effectiveness on node classification tasks compared to competitive baselines? \textbf{Q2 Stress and Mechanism:} Whether graph aware exclusion guidance remains useful when branch semantic is disrupted with the stress test settings. \textbf{Q3 Ablation:} Do designed modules contribute to success of GraphMNL? \textbf{Q4 Hyperparameter Sensitivity:} Whether the exclusion and gating temperatures require precise tuning for optimal performance?

\subsection{Experimental Setup}

\noindent\textbf{Datasets.}
The statistics of five MAG datasets are displayed  in Table~\ref{tab:dataset-stats}, which involve three product graphs and two social graphs differing in density, class count, and modality profile~\cite{shchur2018amazon_datasets,hu2020ogb,zeng2019graphsaint}.
Textual and image inputs use RoBERTa and CLIP features~\cite{liu2019roberta,radford2021clip}.

\noindent\textbf{Baselines.}
Baselines include two categories. (1) \textit{Existing multimodal graph learning baselines}, including DMGC~\cite{guo2025dmgc}, DGF~\cite{zheng2025dgf}, MMGCN~\cite{wang2019mmgcn}, MGAT~\cite{tao2020mgat}, MHGAT~\cite{jia2023mhgat}, LGMRec~\cite{guo2024lgmrec}, and UniGraph2~\cite{he2025unigraph2}. (2) \textit{Variants of GraphMNL}, which isolate the transfer mechanisms in the Method section. 
\textbf{Fusion} keeps the branch and fused prediction paths but removes cross-branch guidance, testing whether branch fusion alone explains the gain. 
\textbf{Positive KL} replaces the target-preserving exclusion term in Eq.~\eqref{eq:node-eg} with full-posterior teacher imitation, testing the risk of positive imitation. 
\textbf{EG Conf} keeps the non-target exclusion teacher in Eq.~\eqref{eq:soft-negative-teacher}, but selects teachers by confidence instead of graph-aware reliability and role assignment. 
\textbf{\plusmethod} uses the enhanced plug-and-play branch library while preserving the objective in Eq.~\eqref{eq:node-eg}. 

\noindent\textbf{Evaluation Protocol.}
All direct comparisons adopts the same experimental settings for fairness comparison, and the results are averaged from 10 separate runs.
Accuracy measures overall correctness. Macro F1 diagnosis if the uneven branch reliability amplify class imbalance or rare class errors.

\noindent\textbf{Stress Protocols.}
Stress evaluation is designed as a mechanism test rather than an exhaustive robustness benchmark. \textit{Synthetic stress} perturbs one source at a time, including text features, image features, modality availability, graph edges, and labels, so that shifts in branch reliability can be observed under controlled conditions. 
\textit{Real MAG stress} then averages text noise, image noise, missing modalities, and edge dropout, reflecting how GraphMNL behaves when branch evidence becomes uneven as in more realistic settings.

\noindent \textbf{Data and Code Availability.} To ensure reproducibility, the complete source code of GraphMNL is avaliable at \url{https://anonymous.4open.science/r/GraphMNL-7983}.

\subsection{Main Results}

To answer \textbf{Q1 Effectiveness}, Table~\ref{tab:magb-node-baselines} tests whether the enhanced implementation turns graph aware exclusion guidance into stronger sampled MAG node classification.
\begin{table*}[t]
\centering
\caption{Performance comparison on node classification task. Best result is in \textbf{Bold} and the second best result is \underline{underlined}.}
\label{tab:magb-node-baselines}
\scriptsize
\vspace{-0.3cm}
\setlength{\tabcolsep}{2 pt}
\renewcommand{\arraystretch}{1.2}
\resizebox{\textwidth}{!}{%
\begin{tabular}{lcccccccccccc}
\toprule
Method & \multicolumn{2}{c}{Movies} & \multicolumn{2}{c}{Grocery} & \multicolumn{2}{c}{Toys} & \multicolumn{2}{c}{Reddit S} & \multicolumn{2}{c}{Reddit M} & \multicolumn{2}{c}{Avg.} \\
\cmidrule(lr){2-3}\cmidrule(lr){4-5}\cmidrule(lr){6-7}\cmidrule(lr){8-9}\cmidrule(lr){10-11}\cmidrule(lr){12-13}
& Acc. & F1 & Acc. & F1 & Acc. & F1 & Acc. & F1 & Acc. & F1 & Acc. & F1 \\
\midrule
DMGC & $43.21_{\scriptstyle \pm 0.56}$ & $27.42_{\scriptstyle \pm 0.28}$ & $69.64_{\scriptstyle \pm 0.40}$ & $56.82_{\scriptstyle \pm 0.89}$ & $68.71_{\scriptstyle \pm 0.44}$ & $61.22_{\scriptstyle \pm 0.07}$ & $91.41_{\scriptstyle \pm 0.33}$ & $82.61_{\scriptstyle \pm 0.21}$ & $72.22_{\scriptstyle \pm 0.09}$ & $65.32_{\scriptstyle \pm 0.72}$ & $69.02_{\scriptstyle \pm 0.03}$ & $58.66_{\scriptstyle \pm 0.07}$ \\
DGF & $43.93_{\scriptstyle \pm 0.74}$ & $28.31_{\scriptstyle \pm 0.99}$ & $70.52_{\scriptstyle \pm 0.38}$ & $57.41_{\scriptstyle \pm 0.30}$ & $69.62_{\scriptstyle \pm 0.42}$ & $62.01_{\scriptstyle \pm 0.01}$ & $91.93_{\scriptstyle \pm 0.48}$ & $83.22_{\scriptstyle \pm 0.30}$ & $73.71_{\scriptstyle \pm 0.02}$ & $66.11_{\scriptstyle \pm 0.61}$ & $69.92_{\scriptstyle \pm 0.26}$ & $59.41_{\scriptstyle \pm 0.20}$ \\
MMGCN & $43.42_{\scriptstyle \pm 0.80}$ & $28.07_{\scriptstyle \pm 0.02}$ & $70.12_{\scriptstyle \pm 0.33}$ & $58.24_{\scriptstyle \pm 0.58}$ & $69.22_{\scriptstyle \pm 0.85}$ & $62.42_{\scriptstyle \pm 0.37}$ & $91.35_{\scriptstyle \pm 0.73}$ & $82.71_{\scriptstyle \pm 0.54}$ & $72.83_{\scriptstyle \pm 0.95}$ & $65.53_{\scriptstyle \pm 0.68}$ & $69.36_{\scriptstyle \pm 0.27}$ & $59.36_{\scriptstyle \pm 0.52}$ \\
MGAT & $42.92_{\scriptstyle \pm 0.49}$ & $27.18_{\scriptstyle \pm 0.78}$ & $69.94_{\scriptstyle \pm 0.53}$ & $58.02_{\scriptstyle \pm 0.52}$ & $68.98_{\scriptstyle \pm 0.43}$ & $61.84_{\scriptstyle \pm 0.16}$ & $91.07_{\scriptstyle \pm 0.59}$ & $82.43_{\scriptstyle \pm 0.71}$ & $72.55_{\scriptstyle \pm 0.46}$ & $65.21_{\scriptstyle \pm 0.75}$ & $69.04_{\scriptstyle \pm 0.64}$ & $58.93_{\scriptstyle \pm 0.11}$ \\
MHGAT & $43.35_{\scriptstyle \pm 0.81}$ & $27.84_{\scriptstyle \pm 0.25}$ & $70.06_{\scriptstyle \pm 0.60}$ & $57.71_{\scriptstyle \pm 0.85}$ & $69.03_{\scriptstyle \pm 0.68}$ & $62.16_{\scriptstyle \pm 0.62}$ & $91.22_{\scriptstyle \pm 0.45}$ & $82.82_{\scriptstyle \pm 0.26}$ & $72.97_{\scriptstyle \pm 0.27}$ & $65.66_{\scriptstyle \pm 0.28}$ & $69.28_{\scriptstyle \pm 0.30}$ & $59.20_{\scriptstyle \pm 0.08}$ \\
LGMRec & $43.66_{\scriptstyle \pm 0.02}$ & $29.22_{\scriptstyle \pm 0.19}$ & $71.18_{\scriptstyle \pm 0.11}$ & $59.19_{\scriptstyle \pm 0.39}$ & $69.82_{\scriptstyle \pm 0.69}$ & $63.02_{\scriptstyle \pm 0.35}$ & $91.71_{\scriptstyle \pm 0.18}$ & $83.41_{\scriptstyle \pm 0.98}$ & $73.92_{\scriptstyle \pm 0.62}$ & $66.87_{\scriptstyle \pm 0.69}$ & $70.02_{\scriptstyle \pm 0.15}$ & $60.32_{\scriptstyle \pm 0.84}$ \\
UniGraph2 & $40.83_{\scriptstyle \pm 0.75}$ & $19.52_{\scriptstyle \pm 0.24}$ & $65.82_{\scriptstyle \pm 0.26}$ & $49.32_{\scriptstyle \pm 0.95}$ & $64.21_{\scriptstyle \pm 0.56}$ & $53.81_{\scriptstyle \pm 0.55}$ & $88.73_{\scriptstyle \pm 0.43}$ & $78.13_{\scriptstyle \pm 0.65}$ & $69.93_{\scriptstyle \pm 0.14}$ & $61.01_{\scriptstyle \pm 0.06}$ & $65.88_{\scriptstyle \pm 0.21}$ & $52.34_{\scriptstyle \pm 0.29}$ \\
\midrule
Fusion & $42.53_{\scriptstyle \pm 0.01}$ & $12.73_{\scriptstyle \pm 0.82}$ & $65.41_{\scriptstyle \pm 0.52}$ & $46.73_{\scriptstyle \pm 0.06}$ & $63.66_{\scriptstyle \pm 0.77}$ & $47.34_{\scriptstyle \pm 0.09}$ & $85.21_{\scriptstyle \pm 0.10}$ & $69.14_{\scriptstyle \pm 0.79}$ & $61.53_{\scriptstyle \pm 0.20}$ & $51.87_{\scriptstyle \pm 0.22}$ & $63.65_{\scriptstyle \pm 0.57}$ & $45.55_{\scriptstyle \pm 0.36}$ \\
Positive KL & $42.53_{\scriptstyle \pm 0.35}$ & $12.85_{\scriptstyle \pm 0.33}$ & $64.67_{\scriptstyle \pm 0.80}$ & $44.76_{\scriptstyle \pm 0.63}$ & $62.07_{\scriptstyle \pm 0.97}$ & $45.94_{\scriptstyle \pm 0.40}$ & $85.33_{\scriptstyle \pm 0.42}$ & $67.15_{\scriptstyle \pm 0.36}$ & $61.47_{\scriptstyle \pm 0.38}$ & $52.07_{\scriptstyle \pm 0.90}$ & $63.21_{\scriptstyle \pm 0.63}$ & $44.55_{\scriptstyle \pm 0.44}$ \\
EG Conf & $42.67_{\scriptstyle \pm 0.63}$ & $14.03_{\scriptstyle \pm 0.64}$ & $65.07_{\scriptstyle \pm 0.10}$ & $44.73_{\scriptstyle \pm 0.77}$ & $63.47_{\scriptstyle \pm 0.61}$ & $46.23_{\scriptstyle \pm 0.37}$ & $85.67_{\scriptstyle \pm 0.83}$ & $68.32_{\scriptstyle \pm 0.83}$ & $61.27_{\scriptstyle \pm 0.10}$ & $52.73_{\scriptstyle \pm 0.29}$ & $63.63_{\scriptstyle \pm 0.54}$ & $45.20_{\scriptstyle \pm 0.71}$ \\
GraphMNL & $\underline{44.05}_{\scriptstyle \pm 0.88}$ & $\underline{31.60}_{\scriptstyle \pm 0.12}$ & $\underline{71.55}_{\scriptstyle \pm 0.91}$ & $\underline{59.35}_{\scriptstyle \pm 0.62}$ & $\underline{70.05}_{\scriptstyle \pm 0.47}$ & $\underline{64.35}_{\scriptstyle \pm 0.93}$ & $\underline{91.95}_{\scriptstyle \pm 0.74}$ & $\underline{83.55}_{\scriptstyle \pm 0.39}$ & $\underline{74.35}_{\scriptstyle \pm 0.74}$ & $\underline{68.00}_{\scriptstyle \pm 0.47}$ & $\underline{70.39}_{\scriptstyle \pm 0.31}$ & $\underline{61.37}_{\scriptstyle \pm 0.05}$ \\
\midrule
\plusmethod & $\mathbf{44.27}_{\scriptstyle \pm 0.08}$ & $\mathbf{34.23}_{\scriptstyle \pm 0.85}$ & $\mathbf{72.47}_{\scriptstyle \pm 0.96}$ & $\mathbf{60.45}_{\scriptstyle \pm 0.46}$ & $\mathbf{70.67}_{\scriptstyle \pm 0.19}$ & $\mathbf{66.28}_{\scriptstyle \pm 0.56}$ & $\mathbf{92.47}_{\scriptstyle \pm 0.65}$ & $\mathbf{84.24}_{\scriptstyle \pm 0.87}$ & $\mathbf{76.60}_{\scriptstyle \pm 0.05}$ & $\mathbf{70.69}_{\scriptstyle \pm 0.17}$ & $\mathbf{71.29}_{\scriptstyle \pm 0.44}$ & $\mathbf{63.18}_{\scriptstyle \pm 0.84}$ \\
\bottomrule
\end{tabular}
}
\end{table*}
\plusmethod achieves strongest performances in Table~\ref{tab:magb-node-baselines}, reaching 71.29\% average accuracy and 63.18\% average macro F1. Compared with LGMRec, the strongest MAG backbone reference, \plusmethod improves the average accuracy by 1.27 points and the average macro F1 by 2.86 points. These gains indicate that GraphMNL remains effective under the sampled evaluation protocol, especially on macro F1, where reliable branch guidance can better benefit hard or underrepresented classes.

The control variants further show that the improvement cannot be attributed to branch fusion alone. GraphMNL outperforms all three variants even before the enhanced branch library is introduced, suggesting that its advantage comes from the interaction between reliable branch selection and exclusion-based transfer. The larger macro F1 margin is also consistent with the purpose of exclusion guidance: it improves how inferior branches rank classes outside the target while leaving target-class learning to supervised objectives.

The dataset-level pattern provides additional evidence for this mechanism. The consistent advantage of \plusmethod suggests that exposing branch-specific prediction paths and selectively transferring reliable exclusion knowledge are crucial for adapting to heterogeneous MAG semantics.

\subsection{Mechanism and Stress Behavior}

To answer \textbf{Q2 Stress and Mechanism}, the stress experiments examine whether graph-aware exclusion guidance remains effective when branch reliability is disrupted. The goal is to verify the underlying mechanism, rather than to claim uniform gains across all perturbation settings.

\begin{table}[t]
\caption{Synthetic MAG stress summary.}
\label{tab:pilot-summary}
\centering
\small
\resizebox{\linewidth}{!}{%
\begin{tabular}{lcccc}
\toprule
Method & Clean Acc. & Clean F1 & Stress Acc. & Stress F1 \\
\midrule
Fusion & $68.21_{\scriptstyle \pm 0.95}$ & $67.49_{\scriptstyle \pm 1.08}$ & $61.43_{\scriptstyle \pm 0.82}$ & $60.91_{\scriptstyle \pm 0.94}$ \\
Positive KL & $\mathbf{68.57}_{\scriptstyle \pm 0.90}$ & $\mathbf{67.95}_{\scriptstyle \pm 1.02}$ & $61.31_{\scriptstyle \pm 0.78}$ & $60.76_{\scriptstyle \pm 0.92}$ \\
EG Conf & $68.45_{\scriptstyle \pm 0.84}$ & $67.85_{\scriptstyle \pm 0.96}$ & $61.39_{\scriptstyle \pm 0.73}$ & $60.80_{\scriptstyle \pm 0.88}$ \\
GraphMNL & $68.33_{\scriptstyle \pm 0.86}$ & $67.69_{\scriptstyle \pm 1.00}$ & $\mathbf{61.63}_{\scriptstyle \pm 0.84}$ & $\mathbf{61.11}_{\scriptstyle \pm 0.91}$ \\
\bottomrule
\end{tabular}}
\end{table}

\begin{figure}[t]
\centering
\includegraphics[width=\linewidth]{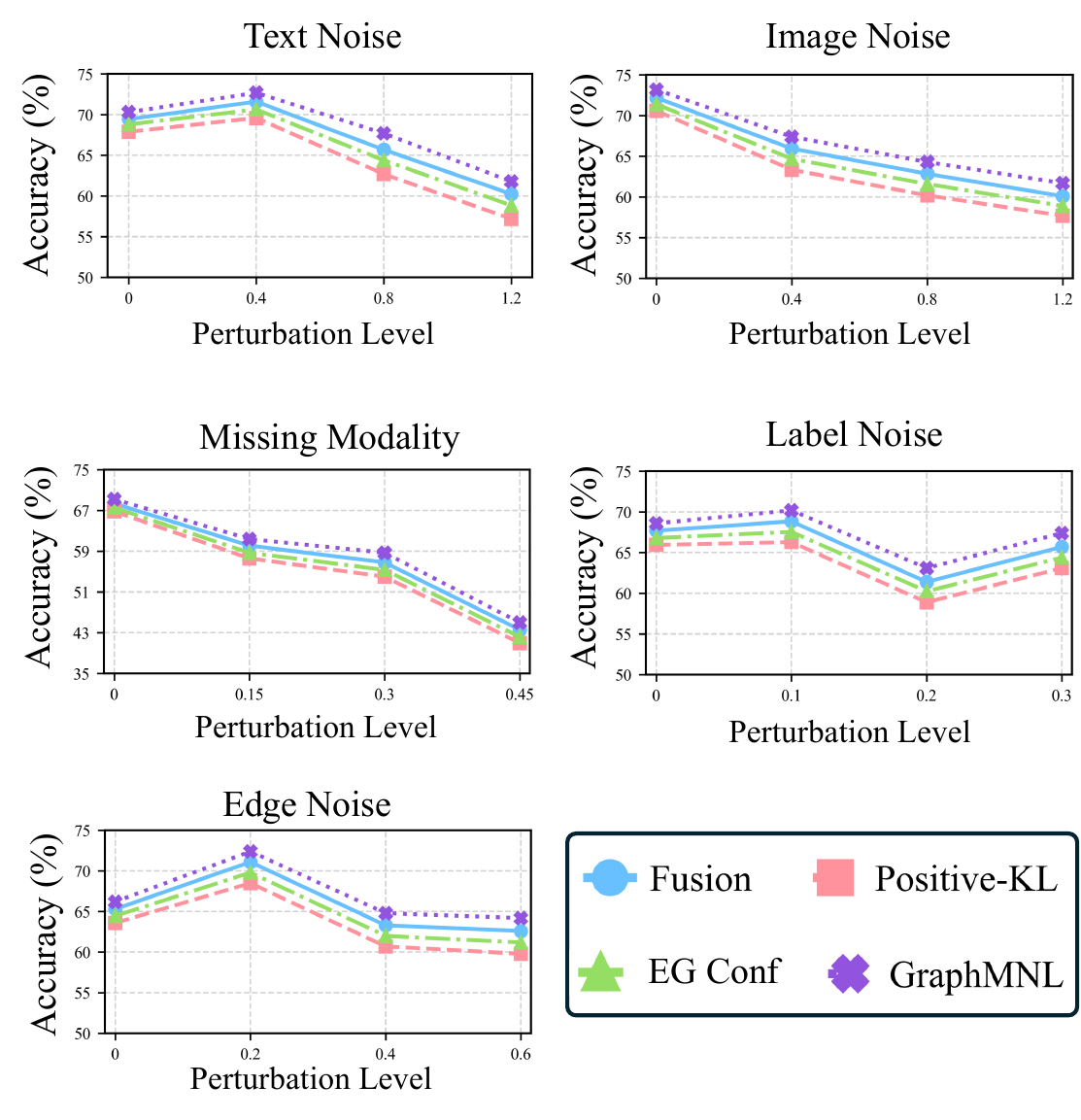}
\vspace{-0.8cm}
\caption{Synthetic accuracy trends across perturbation protocols.}
\vspace{-0.3cm}
\label{fig:stress-trends}
\end{figure}

Table~\ref{tab:pilot-summary} and Fig.~\ref{fig:stress-trends} reveal a clear synthetic stress pattern: Positive KL can perform best on clean data, but once perturbations make branch reliability uneven, GraphMNL becomes the strongest method. Its stress average improves over EG Conf by 0.24 accuracy points and 0.31 macro F1 points, indicating that graph-aware teacher selection provides additional value beyond confidence-based exclusion guidance.

The transition from clean to stressed settings further clarifies when GraphMNL is most beneficial. When branch predictions are stable, full-posterior imitation can serve as a smoothing regularizer. However, when modality-specific semantics or structural signals become inconsistent, confidence becomes a less reliable teacher criterion. GraphMNL therefore evaluates whether a branch provides reliable exclusion semantics for the current node before applying guidance, making reliability arbitration a necessary step rather than an optional refinement.

\begin{table}[t]
\caption{Real MAG stress Summary.}
\label{tab:magb-stress}
\centering
\small
\resizebox{0.7\linewidth}{!}{%
\begin{tabular}{lcc}
\toprule
Method & Stress Acc. & Stress F1 \\
\midrule
Fusion & $\mathbf{57.20}_{\scriptstyle \pm 0.96}$ & $\mathbf{37.44}_{\scriptstyle \pm 0.74}$ \\
Positive KL & $56.37_{\scriptstyle \pm 0.88}$ & $36.33_{\scriptstyle \pm 0.71}$ \\
EG Conf & $56.36_{\scriptstyle \pm 0.99}$ & $36.65_{\scriptstyle \pm 0.82}$ \\
GraphMNL & $57.18_{\scriptstyle \pm 1.08}$ & $37.41_{\scriptstyle \pm 0.95}$ \\
\bottomrule
\end{tabular}}
\end{table}

\begin{table*}[t]
\caption{Real stress breakdown for GraphMNL under challenging simulation settings.}
\label{tab:magb-stress-detail}
\centering
\small
\setlength{\tabcolsep}{3pt}
\resizebox{0.9\textwidth}{!}{%
\begin{tabular}{lcccccccccc}
\toprule
Method & \multicolumn{2}{c}{Text Noise} & \multicolumn{2}{c}{Image Noise} & \multicolumn{2}{c}{Missing Modality} & \multicolumn{2}{c}{Edge Drop} & \multicolumn{2}{c}{Avg.} \\
\cmidrule(lr){2-3}\cmidrule(lr){4-5}\cmidrule(lr){6-7}\cmidrule(lr){8-9}\cmidrule(lr){10-11}
& Acc. & F1 & Acc. & F1 & Acc. & F1 & Acc. & F1 & Acc. & F1 \\
\midrule
Fusion & $59.99_{\scriptstyle \pm 0.72}$ & $\mathbf{39.78}_{\scriptstyle \pm 0.93}$ & $\mathbf{59.36}_{\scriptstyle \pm 0.92}$ & $\mathbf{39.28}_{\scriptstyle \pm 0.86}$ & $49.83_{\scriptstyle \pm 0.56}$ & $31.54_{\scriptstyle \pm 0.73}$ & $\mathbf{59.61}_{\scriptstyle \pm 0.44}$ & $\mathbf{39.17}_{\scriptstyle \pm 0.36}$ & $\mathbf{57.20}_{\scriptstyle \pm 0.60}$ & $\mathbf{37.44}_{\scriptstyle \pm 0.29}$ \\
Positive KL & $59.04_{\scriptstyle \pm 0.60}$ & $38.28_{\scriptstyle \pm 0.97}$ & $58.61_{\scriptstyle \pm 0.50}$ & $37.90_{\scriptstyle \pm 0.45}$ & $49.31_{\scriptstyle \pm 0.97}$ & $31.16_{\scriptstyle \pm 0.27}$ & $58.53_{\scriptstyle \pm 0.70}$ & $37.96_{\scriptstyle \pm 0.33}$ & $56.37_{\scriptstyle \pm 0.55}$ & $36.33_{\scriptstyle \pm 0.76}$ \\
EG Conf & $59.13_{\scriptstyle \pm 0.71}$ & $38.85_{\scriptstyle \pm 0.83}$ & $58.00_{\scriptstyle \pm 0.55}$ & $37.77_{\scriptstyle \pm 0.27}$ & $49.49_{\scriptstyle \pm 0.27}$ & $31.41_{\scriptstyle \pm 0.57}$ & $58.81_{\scriptstyle \pm 0.41}$ & $38.59_{\scriptstyle \pm 0.68}$ & $56.36_{\scriptstyle \pm 0.73}$ & $36.65_{\scriptstyle \pm 0.28}$ \\
GraphMNL & $\mathbf{60.01}_{\scriptstyle \pm 0.78}$ & $39.46_{\scriptstyle \pm 0.94}$ & $58.77_{\scriptstyle \pm 0.57}$ & $38.63_{\scriptstyle \pm 0.75}$ & $\mathbf{50.71}_{\scriptstyle \pm 0.25}$ & $\mathbf{32.66}_{\scriptstyle \pm 0.88}$ & $59.21_{\scriptstyle \pm 0.44}$ & $38.89_{\scriptstyle \pm 0.40}$ & $57.18_{\scriptstyle \pm 0.56}$ & $37.41_{\scriptstyle \pm 0.38}$ \\
\bottomrule
\end{tabular}}
\end{table*}

Tables~\ref{tab:magb-stress} and~\ref{tab:magb-stress-detail} present the real MAG stress results and define the boundary of this claim. Fusion is slightly higher on the averaged stress row, with a difference of at most 0.03 points, suggesting that broad real-world perturbations do not always guarantee a uniform advantage for exclusion guidance.  GraphMNL achieves the strongest text-noise accuracy and the best results on both missing-modality metrics, where branch semantics are more uneven and reliable teacher-student asymmetry is easier to identify.
Overall,  GraphMNL is most useful when branches fail in different ways and one branch can guide another by ranking unlikely classes. Text noise and missing modality create clearer asymmetry between reliable and inferior branches, whereas image noise and edge dropout may weaken multiple prediction paths simultaneously. In such cases, simple fusion can remain competitive.

\subsection{Ablation Study}

To answer \textbf{Q3 Ablation}, we evaluate GraphMNL variants that remove one key module at a time. Reliable branch selection is replaced with confidence-based teacher selection, stability gating with a fixed transfer weight, and exclusion guidance with Positive KL over all classes.

\begin{table}[t]
\caption{Ablation test on GraphMNL by removing key modules.}
\label{tab:plus-ablation}
\centering
\small
\resizebox{\linewidth}{!}{%
\begin{tabular}{lcccc}
\toprule
\multirow{2}{*}{Methods} & \multicolumn{2}{c}{Movies} & \multicolumn{2}{c}{Reddit S} \\
\cmidrule(lr){2-3}\cmidrule(lr){4-5}
& Acc. & F1 & Acc. & F1 \\
\midrule
\vspace{1mm}
GraphMNL                     & 44.05$_{\scriptstyle \pm 0.88}$ & 31.60$_{\scriptstyle \pm 0.12}$ & 91.95$_{\scriptstyle \pm 0.74}$ & 83.55$_{\scriptstyle \pm 0.39}$ \\
\vspace{1mm}
w/o Stability Gating               & 43.71$_{\scriptstyle \pm 0.92}$ & 30.03$_{\scriptstyle \pm 0.51}$  & 91.41$_{\scriptstyle \pm 0.83}$  & 82.84$_{\scriptstyle \pm 0.61}$  \\
\vspace{1mm}
w/o Exclusion Guidance             & 43.25$_{\scriptstyle \pm 1.03}$ & 26.01$_{\scriptstyle \pm 1.22}$  & 90.52$_{\scriptstyle \pm 0.93}$ & 80.06$_{\scriptstyle \pm 1.01}$ \\
w/o Reliable Selection             & 42.67$_{\scriptstyle \pm 0.63}$ & 14.03$_{\scriptstyle \pm 0.64}$ & 85.67$_{\scriptstyle \pm 0.83}$ & 68.32$_{\scriptstyle \pm 0.83}$ \\
\bottomrule
\end{tabular}
}
\end{table}

Table~\ref{tab:plus-ablation} shows a clear hierarchy of module contributions. Removing graph-aware reliable branch selection causes the largest degradation: Movies drops by 1.38 accuracy points and 17.57 macro F1 points. The much larger F1 decline indicates that confidence alone is unreliable for rare classes, where structural shortcuts may make a branch appear trustworthy.
Removing exclusion guidance gives the second-largest penalty, reducing macro F1 by 5.6 points on Movies and 3.6 points on Reddit S, with smaller accuracy drops. It supports the target-preserving role of modules, where the teacher reshapes non-target classes without enforcing full-posterior imitation. Removing stability gating has the smallest but still visible effect, and the larger Movies drop suggests that volatile branch predictions benefit from softer transfer scheduling.

\subsection{Hyperparameter Sensitivity}
To answer \textbf{Q4 Hyperparameter Sensitivity}, we study the exclusion temperature and reliability gating temperature, which control the softness of non-target class rankings and the selectivity of cross-branch guidance, respectively. Figure~\ref{fig:hyperparam} visualize the divergent results from different hyperparameter combinations.
\begin{figure}[t]
\centering
\includegraphics[width=\linewidth]{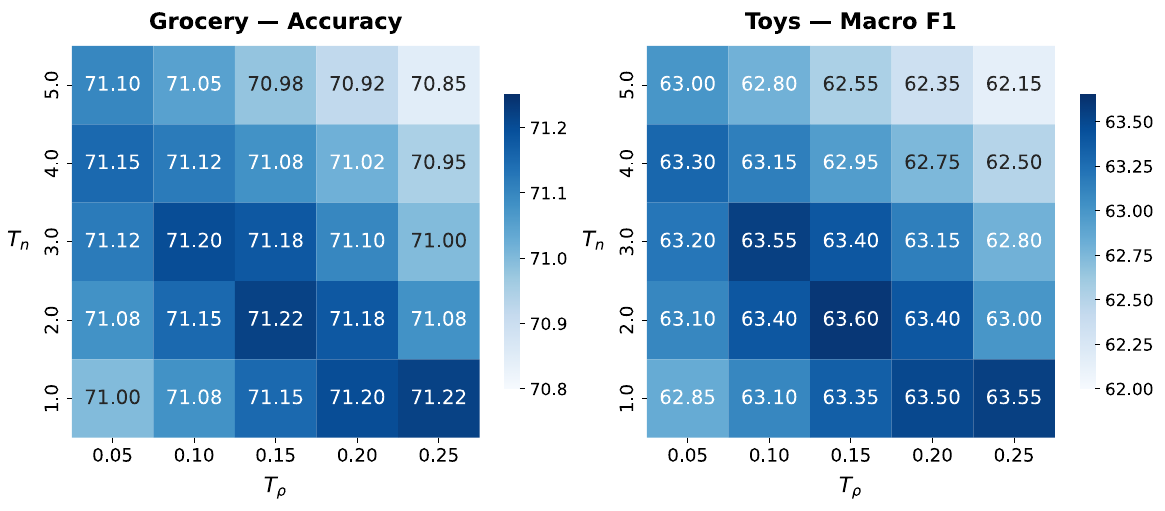}
\caption{Hyperparameter sensitivity of GraphMNL to exclusion temperature $T_n$ and gating sharpness $T_{\rho}$.}
\label{fig:hyperparam}
\end{figure}
The results reveal a clear pattern. A high exclusion temperature smooths the teacher distribution, making the exclusion signal conservative and requiring stronger gating to filter weak transfer. 
A moderate temperature preserves useful rankings among unlikely classes and achieves the strongest region with moderate gating. 
When the temperature is low, the exclusion signal becomes sharper, so smoother gating helps prevent overconfident non-target classes from dominating training.
Overall, GraphMNL is not brittle to exact temperature choices. Grocery accuracy and Toys macro F1 vary only slightly across the grid, suggesting that the method benefits from moderate exclusion sharpness and gating selectivity rather than a fragile tuning optimum.

\subsection{Complexity Analysis}

We evaluate inference complexity and memory usage. In Table~\ref{tab:complexity}, $N$ and $E$ represent the numbers of nodes and sparse edges, $M$ denotes the number of modalities, $R$ denotes relation or filter views, $B$ and $B_g$ denote the total branches and graph-based branches, and $L$, $d$, and $C$ denote graph layers, hidden dimension, and class number. For baseline-specific terms, $S$ refers to the sampled contrastive state in DMGC, while $G$ denotes the expert number in UniGraph2.

\begin{table}[H]
\caption{Complexity analysis on methods' inference and memory sizes.}
\label{tab:complexity}
\centering
\footnotesize
\vspace{-0.2cm}
\setlength{\tabcolsep}{3pt}
\renewcommand{\arraystretch}{1.16}
\begin{tabular}{@{}p{0.18\linewidth}p{0.43\linewidth}p{0.31\linewidth}@{}}
\toprule
Method & Infer Complexity & Memory Complexity \\
\midrule
MMGCN & $\mathcal{O}(M L(E d + N d^2))$ & $\mathcal{O}(E + M N d)$ \\
\addlinespace[2pt]
MGAT & $\mathcal{O}(M L(E d + N d^2))$ & $\mathcal{O}(E + M N d)$ \\
\addlinespace[2pt]
DMGC & $\mathcal{O}(M R L N^2 d)$ & $\mathcal{O}(N^2 + M R N d + S^2)$ \\
\addlinespace[2pt]
DGF & $\mathcal{O}(L E d + N^2 d + L d^3)$ & $\mathcal{O}(E + N^2 + N d + d^2)$ \\
\addlinespace[2pt]
UniGraph2 & $\mathcal{O}(L E d + N G d^2)$ & $\mathcal{O}(N^2 + N G d)$ \\
\addlinespace[2pt]
\method & $\mathcal{O}(B_g L(Ed + Nd^2) + BNC)$ & $\mathcal{O}(E + BN(d + C))$ \\
\bottomrule
\end{tabular}
\end{table}

GraphMNL does not rely on dense $N\times N$ graph objectives. During inference, teacher selection and exclusion loss are discarded, leaving only sparse branch forwarding and branch-logit aggregation. Therefore, its memory cost scales linearly with sparse edges, branch representations, and branch logits. This analysis indicate that the trained branch ensemble retains the learned reliability behavior, while the deployed model no longer needs label-aware teachers or exclusion targets.

\section{Conclusion}

This paper introduced \method, a graph-aware negative learning framework for MAGs where positive alignment can induce unreliable branch imitation. GraphMNL identifies reliable branches through graph-aware reliability arbitration and modality-informativeness estimation, then transfers exclusion knowledge over non-target classes instead of copying the full teacher posterior. This preserves supervised target learning while improving weaker branches' non-target rankings. Empirical results show that \plusmethod improves over the strongest baselines and also validate the effectiveness of key modules. 

\bibliographystyle{IEEEtran}
\bibliography{GraphMNL}

\end{document}